\documentclass{article}

\usepackage[preprint]{neurips_2019}

\usepackage[utf8]{inputenc} 
\usepackage[T1]{fontenc}    
\usepackage{hyperref}       
\usepackage{url}            
\usepackage{booktabs}       
\usepackage{amsfonts}       
\usepackage{nicefrac}       
\usepackage{microtype}      

\title{On the Expressive Power of \\ Deep Polynomial Neural Networks}

\author{%
Joe Kileel\thanks{Equal contribution.} \\
  Princeton University\\
  \And
  Matthew Trager\footnotemark[1]\\
  New York University\\
  \And
  Joan Bruna\\
  New York University\\
  }

\usepackage{amsmath, amsthm, amssymb}
\usepackage{graphicx}
\usepackage{tikz}
\usepackage{bm}
\usepackage{thmtools}
\usepackage{thm-restate}
\usepackage{cleveref}
\usepackage{float}
\floatstyle{plaintop}
\restylefloat{table}
\usepackage[tableposition=top]{caption}

\newtheorem{theorem}{Theorem}

\newtheorem{definition}[theorem]{Definition}

\newtheorem{remark}[theorem]{Remark}
\newtheorem{example}[theorem]{Example}
\newtheorem{conjecture}[theorem]{Conjecture}

\newtheorem*{claim*}{Claim}

\newcommand{\RR}{\mathbb{R}}

\newcommand{\PP}{\mathbb{P}}
\newcommand{\CC}{\mathbb{C}}
\newcommand{\ZZ}{\mathbb{Z}}

\newcommand{\FF}{\mathbb{F}}

\newcommand{\Sym}[2]{\textup{Sym}_{#1}(\RR^{#2})}
\newcommand{\V}[2]{\mathcal V_{\bm #1, #2}}

\newcommand{\F}[2]{\mathcal F_{\bm #1, #2}}

\newcommand{\eg}{{\em e.g.}}
\newcommand{\ie}{{\em i.e.}}

\begin{document}

\maketitle
\setcounter{footnote}{0}

\begin{abstract} We study deep neural networks with polynomial activations, particularly \textit{their expressive power}.  For a fixed architecture and activation degree, a polynomial neural network defines an algebraic map from weights to polynomials.  The image of this map is the functional space associated to the network, and it is an irreducible algebraic variety upon taking closure.  This paper proposes \textit{the dimension of this variety} as a precise measure of the expressive power of polynomial neural networks.  We obtain several theoretical results regarding this dimension as a function of architecture, including an exact formula for high activation degrees, as well as upper and lower bounds on layer widths in order for deep polynomials networks to fill the ambient functional space. We also present computational evidence that it is profitable in terms of expressiveness for layer widths to increase monotonically and then decrease monotonically.  Finally, we link our study to favorable optimization properties when training weights, and we draw  intriguing connections with tensor and polynomial decompositions.  
\end{abstract}

\section{Introduction}

A fundamental problem in the theory of deep learning is to study the \emph{functional space} of deep neural networks. A network can be modeled as a composition of elementary maps, however the family of all functions that can be obtained in this way is extremely complex. Many recent papers paint an accurate picture for the case of {shallow networks} (\eg, using mean field theory~\cite{chizat_global_2018, mei_mean_2018}) and of {deep linear networks}~\cite{arora_convergence_2018,arora_optimization_2018,kawaguchi_deep_2016}, however a similar investigation of \emph{deep nonlinear} networks appears to be significantly more challenging, and require very different tools.

In this paper, we consider a general model for \emph{deep polynomial neural networks}, where the activation function is a polynomial ($r$-th power) exponentiation. The advantage of this framework is that the functional space associated with a network architecture is \emph{algebraic}, so we can use tools from \textit{algebraic geometry}~\cite{harris_algebraic_1995} for a precise investigation of deep neural networks. Indeed, for a fixed activation degree $r$ and architecture $\bm d =(d_0,\ldots,d_h)$ (expressed as a sequence of widths), the family of all networks with varying weights can be identified with an \emph{algebraic variety} $\V{d}{r}$, embedded in a finite-dimensional Euclidean space. In this setting, an algebraic variety can be thought of as a manifold that may have singularities.

In this paper, our main object of study is the \emph{dimension} of $\V{d}{r}$ as a variety (in practice, as a manifold), which may be regarded as a precise measure of the architecture's expressiveness. Specifically, we prove that this dimension stabilizes when activations are high degree, and we provide an exact dimension formula for this case 
(Theorem \ref{thm:naive}). 
We also investigate conditions under which $\V{d}{r}$ \emph{fills} its ambient space. This question is important from the vantage point of optimization, since an architecture is ``filling'' if and only if it corresponds to a convex functional space (Proposition~\ref{prop:opt1}). In this direction, we prove a \textit{bottleneck property}, that if a width is not sufficiently large, the network can never fill the ambient space regardless of the size of other layers (Theorem~\ref{thm:bottleneck}).  

In a broader sense, our work introduces a powerful language and suite of mathematical tools for studying the geometry of network architectures. Although this setting requires polynomial activations, it may be used as a testing ground for more general situations and, \eg, to verify rules of thumb rigorously. Finally, our results show that polynomial neural networks are intimately related to the theory of \emph{tensor decompositions}~\cite{Landsberg-book}. In fact, representing a polynomial as a deep network corresponds to a type of decomposition of tensors which may be viewed as a composition of decompositions of a recently introduced sort \cite{LORS-2019}.  Using this connection, we establish general non-trivial upper bounds on filling widths (Theorem~\ref{thm:boundFilling}).
We believe that our work can serve as a step towards many interesting research challenges in developing the theoretical underpinnings of deep learning.

\subsection{Related work}

The study of the expressive power of neural networks dates back to seminal work on the universality of networks as function approximators~\cite{cybenko_approximation_1989,hornik_multilayer_1989}. More recently, there has been research supporting the hypothesis of ``depth efficiency'', \ie, the fact that deep networks can approximate functions more efficiently than shallow networks~\cite{delalleau2011shallow,martens_expressive_2014,cohen_expressive_2016,cohen_convolutional_2016}. Our paper differs from this line of work, in that we do not emphasize approximation properties, but rather the study of the functions that can be expressed \emph{exactly} using a network.

Most of the aforementioned studies make strong hypotheses on the network architecture. In particular,~\cite{delalleau2011shallow,martens_expressive_2014} focus on \emph{arithmetic circuits}, or \emph{sum-product  networks}~\cite{poon_sum-product_2012}. These are networks composed of units that compute either the product or a weighted sum of their inputs. In~\cite{cohen_expressive_2016}, the authors introduce a model of \emph{convolutional arithmetic circuits}. This is a particular class of arithmetic circuits that includes networks with layers of 1D convolutions and product pooling. This model does not allow for non-linear activations (beside the product pooling), although the follow-up paper~\cite{cohen_convolutional_2016} extends some results to ReLU activations with sum pooling. Interestingly, these networks are  related to Hierarchical Tucker (HT) decomposition of tensors.

The polynomial networks studied in this paper are not arithmetic circuits, but feedforward deep networks with polynomial $r$-th power activations. This is a vast generalization of a setting considered in several recent papers~\cite{venturi2018a,du_power_2018,soltanolkotabi_theoretical_2018}, that study shallow (two layer) networks with quadratic activations ($r=2$). These papers show that if the width of the intermediate layer is at least twice the input dimension, then the quadratic loss has no ``bad'' local minima. This result in line with our Proposition~\ref{prop:CvR}, which explains in this case the functional space is convex and \emph{fills} the ambient space. We also point out that polynomial activations are required for the functional space of the network to span a finite dimensional vector space~\cite{leshno_multilayer_1993,venturi2018a}.

The polynomial networks considered in this paper do not correspond to HT tensor decompositions as in~\cite{cohen_expressive_2016,cohen_convolutional_2016}, rather they are related to a different polynomial/tensor decomposition attracting very recent interest~\cite{FOS-PNAS, LORS-2019}.  These generalize usual decompositions, however their algorithmic and theoretical understanding are, mostly, wide open.  Neural networks motivate several questions in this vein.

\paragraph{Main contributions.}
Our main contributions can be summarized as follows.
\begin{itemize}
    \item We give a precise formulation of the expressiveness of polynomial networks in terms of the algebraic dimension of the functional space as an \emph{algebraic variety}. 
    \item We spell out the close, two-way relationship between polynomial networks and a particular family of decompositions of tensors.
    \item We prove several theoretical results on the functional space of polynomial networks. Notably, we give a formula for the dimension that holds for sufficiently high activation degrees (Theorem~\ref{thm:naive}) and we prove a tight lower bound on the width of the layers for the network to be ``filling'' in the functional space (Theorem~\ref{thm:bottleneck}).
\end{itemize}

\paragraph{Notation.}
We use $\Sym{d}{n}$ to denote the space of \emph{homogeneous polynomials} of degree $d$ in $n$ variables with coefficients in $\RR$. This set is a vector space over $\RR$ of dimension $N_{d,n} = \binom{n+d-1}{d}$, spanned by all monomials of degree $d$ in $n$ variables. In practice, $\Sym{d}{n}$ is isomorphic to $\RR^{N_{d,n}}$, and our networks will correspond to \emph{points} in this high dimensional space. The notation $\Sym{d}{n}$ expresses the fact that a polynomial of degree $d$ in $n$ variables can always be identified with a \emph{symmetric tensor} in $(\RR^n)^{\otimes d}$ that collects all of its coefficients.

\section{Basic setup}

A \emph{polynomial network} is a function $p_{\theta}: \RR^{d_0} \rightarrow \RR^{d_{h}}$ of the form
\begin{equation}
\begin{aligned}
    &p_{\theta}(x) = W_{h} \rho_r W_{h-1} \rho_r \ldots \rho_r W_1 x,
    \quad \,\, W_{i} \in \RR^{d_i \times d_{i-1}},
\end{aligned}
\end{equation}

where the \emph{activation} $\rho_r(z)$ raises all elements of $z$ to the $r$-th power ($r \in \mathbb{N}$). The parameters ${\theta} = (W_{h},\ldots,W_1) \in \RR^{d_{\theta}}$ (with $d_{\theta} = \sum_{i=1}^{h} d_i d_{i-1}$) are the network's \emph{weights}, and the network's \emph{architecture} is encoded by the sequence $\bm d=(d_{0},\ldots,d_{h})$ (specifying the \emph{depth} $h$ and  \emph{widths} $d_i$). Clearly, $p_{\theta}$ is a homogeneous polynomial mapping $\RR^{d_0} \rightarrow \RR^{d_{h}}$ of degree $r^{h-1}$, \ie, $p_{\theta} \in \Sym{r^{h-1}}{d_0}^{d_h}$. 

For fixed degree $r$ and architecture $\bm d=(d_{0},\ldots,d_{h})$, there exists an algebraic map
\begin{equation}\label{eq:phi_map}
    \Phi_{\bm d,r}:  {\theta} \mapsto p_{\theta} = 
    \begin{bmatrix}
    p_{{\theta}1}\\
    \vdots\\
    p_{{\theta} d_{h+1}}\\
    \end{bmatrix},
\end{equation}

where each $p_{{\theta} i}$ is a polynomial in $d_0$ variables. 
The image of $\Phi_{\bm d, r}$ is a set of vectors of polynomials, \ie, a subset $\mathcal F_{\bm d,r}$ of $\Sym{r^{h-1}}{d_0}^{d_h}$, and it is the \emph{functional space} represented by the network. 
In this paper, we consider the ``Zariski closure'' $\V{d}{r}=\overline{\mathcal F_{\bm d,r}}$ of the functional space.\footnote{The Zariski closure of a set $X$ is the smallest set containing $X$ that can be described by polynomial equations.} We refer to $\V{d}{r}$ as \emph{functional variety} of the network architecture, as it is in fact an irreducible \emph{algebraic variety}. In particular, $\V{d}{r}$ can be studied using powerful machinery from \emph{algebraic geometry}.

\vspace{.5em}

\begin{remark}\label{rem:zariski} The functional variety $\V{d}{r}$ may be significantly larger than the actual functional space ${\mathcal F_{\bm d,r}}$, since the Zariski closure is typically larger than the closure with respect to the standard the Euclidean topology. On the other hand, the \emph{dimensions} of
the spaces $\V{d}{r}$ and $\mathcal F_{\bm d,r}$ agree, and the set $\V{d}{r}$ is usually ``nicer'' (it can be described by polynomial equations, whereas an exact implicit description of $\F{d}{r}$ may require inequalities).
\end{remark}

\subsection{Examples}

We present some examples that describe the functional variety $\V{d}{r}$ in simple cases.

\begin{example}
A \emph{linear network} is a polynomial network with $r=1$. In this case, the network map $\Phi_{\bm d, r}: \RR^{d_{\theta}} \rightarrow \Sym{1}{d_0}^{d_h} \cong \RR^{d_h \times d_0}$ is simply matrix multiplication:
\begin{equation}
\theta = (W_h, W_{h-1}, \ldots, W_1) \mapsto p_{\theta} = W_h W_{h-1} \ldots W_{1} x. 
\end{equation}
The functional space $\mathcal{F}_{\bm d, r} \subseteq \RR^{d_{h} \times d_{0}}$ is the set of matrices with rank at most $d_{\min} = \min_i \{d_i\}$.  
This set is already characterized by polynomial equations, as the common zero set of all $(1+d_{\min}) \times (1+d_{\min})$ minors, so $\mathcal{F}_{\bm d, r} = \mathcal{V}_{\bm d, r}$ in this case.  
The dimension of $\mathcal{V}_{\bm d, r} \subset \RR^{d_{h} \times d_{0}}$ is $d_{\min}(d_0 + d_h - d_{\min})$.
\end{example}

\begin{example}
Consider $\bm{d}=(2,2,3)$ and $r=2$.  The input variables are $x = [x_1, x_2]^{T}$, and the parameters $\theta$ are the weights
\begin{equation}
W_1 = \begin{bmatrix}
w_{111} & w_{112} \\[0.15cm]
w_{121} & w_{122}
\end{bmatrix}, \hspace{0.15cm}
W_2 = \begin{bmatrix}
w_{211} & w_{212} \\[0.15cm]
w_{221} & w_{222} \\[0.15cm]
w_{231} & w_{232}
\end{bmatrix}.
\end{equation}
The network map $p_{\theta}$ 
is a triple of quadratic polynomials in $x_1, x_2$,
that can be written as
\begin{equation}\label{eq:network_example}
W_2 \rho_2 W_1 x \hspace{0.4em} = \hspace{0.32em}
\begin{bmatrix}
w_{211}(w_{111}x_{1} + w_{112}x_{2})^2 + w_{212}(w_{121}x_{1} + w_{122}x_{2})^2 \\[0.15cm]
w_{221}(w_{111}x_{1} + w_{112}x_{2})^2 + w_{222}(w_{121}x_{1} + w_{122}x_{2})^2\\[0.15cm]
w_{231}(w_{111}x_{1} + w_{112}x_{2})^2 + w_{232}(w_{121}x_{1} + w_{122}x_{2})^2
\end{bmatrix}.
\end{equation}

The map $\Phi_{(2,2,3), 2}$ in~\eqref{eq:phi_map} takes $W_1,W_2$ (that have $d_{\theta} = 10$ parameters) to the three quadratics in $x_1, x_2$ displayed above. The quadratics have a total of $\dim \Sym{2}{2}^3 = 9$ coefficients, however these coefficients are not arbitrary, \ie, not all possible triples of polynomials occur in the functional space. Writing $c_{ij}^{(k)}$ for the coefficient of $x_ix_j$ in $p_{\theta k}$ in~\eqref{eq:network_example} (with $k=1,2,3$ $i,j = 1,2$) then it is a simple exercise to show that
\begin{equation}
    \det \begin{bmatrix}
    c_{11}^{(1)} & c_{12}^{(1)} & c_{22}^{(1)}\\[0.4em]
    c_{11}^{(2)} & c_{12}^{(2)} & c_{22}^{(2)}\\[0.4em]
    c_{11}^{(3)} & c_{12}^{(3)} & c_{22}^{(3)}\\
    \end{bmatrix} = 0.
\end{equation}

This cubic equation describes the functional variety $\mathcal{V}_{(2,3,3),2}$, which is in this case an eight-dimensional subset (hypersurface) of $\Sym{2}{2}^3 \cong \RR^9$.
\end{example}

\subsection{Objectives}
The main goal of this paper is to study the~\emph{dimension} of $\V{d}{r}$ as the network's architecture $\bm d$ and the activation degree $r$ vary.  This dimension may be considered a \textit{precise} and \textit{intrinsic} measure of the polynomial network's \textit{expressivity}, quantifying degrees of freedom of the functional space.  For example, the dimension reflects the number of input/output pairs the network can interpolate, as each sample imposes one linear constraint on the variety $\mathcal{V}_{\bm d, r}$.

In general, the variety $\V{d}{r}$ lives in the ambient space $\Sym{r^{h-1}}{d_0}^{d_h}$, which in turn only depends on the activation degree $r$, network depth $h$, and the input/output dimensions $d_0$ and $d_h$. We are thus interested in the role of the intermediate widths in the dimension of $\V{d}{r}$.

\begin{definition} A network architecture $\bm d = (d_0,\ldots,d_h)$ has a \emph{filling functional variety} for the activation degree $r$ if $\V{d}{r} = \Sym{r^{h-1}}{d_0}^{d_h}$. 
\end{definition}

It is important to note that if the functional variety $\V{d}{r}$ is filling, then actual functional space $\mathcal F_{\bm d,r}$ (before taking closure) is in general only \emph{thick}, \ie, it has positive Lebesgue measure in $\Sym{r^{h-1}}{d_0}^{d_h}$ (see Remark~\ref{rem:zariski}).
On the other hand, given an architecture with a thick functional space, we can find another architecture whose functional space is the whole ambient space.

\begin{restatable}[Filling functional space]{proposition}{CvR}
\label{prop:CvR}
Fix $r$ and suppose  $\bm d = (d_0,d_1,\ldots,d_{h-1}, d_{h})$ has a filling functional variety $\V{d}{r}$.  Then the architecture $\bm d' = (d_0,2d_1,\ldots,2d_{h-1}, d_{h})$ has a \emph{filling functional space}, \ie, $\F{d'}{r} = \Sym{r^{h-1}}{d_0}^{d_h}$.
\end{restatable}

In summary, while an architecture with a filling functional variety may not necessarily have a filling functional space, it is sufficient to double all the intermediate widths for this stronger condition to hold.
As argued below, we expect  architectures with thick/filling functional spaces to have more favorable properties in terms of optimization and training. On the other hand, non-filling architectures may lead to interesting functional spaces for capturing patterns in data. In fact, we show in Section~\ref{sec:deep_tensors} that non-filling architectures generalize families of low-rank tensors.

\subsection{Connection to optimization}

The following two results illustrate that thick/filling functional spaces are helpful for optimization.

\begin{restatable}{proposition}{firstOpt}
\label{prop:opt1}
If the closure of a set $C \subset \RR^n$ is not convex, then there exists a convex function $f$ on $\RR^n$ whose restriction to $C$ has arbitrarily ``bad'' local minima (that is, there exist local minima whose value is arbitrarily larger than that of a global minimum). 
\end{restatable}

\begin{restatable}{proposition}{secondOpt}
\label{prop:opt2}
If a functional space $\F{d}{r}$ is not thick, then it is not convex. 
\end{restatable}

These two facts show that if the functional space is not thick, we can always find a convex loss function and a data distribution that lead to a landscape with arbitrarily bad local minima. 
There is also an obvious weak converse, namely that if the functional space is filling  $\F{d}{r} = \Sym{r^{h-1}}{d_0}^{d_h}$, then any convex loss function $\F{d}{r}$ will have a unique global minimum (although there may be ``spurious'' critical points that arise from the non-convex parameterization).

\section{Architecture dimensions}

In this section, we begin our study of the dimension of $\V{d}{r}$. We describe the connection between polynomial networks and tensor decompositions for both shallow (Section~\ref{sec:shallow_networks}) and deep (Section~\ref{sec:deep_tensors}) networks, and we present some computational examples (Section~\ref{sec:computations}).

\subsection{Shallow networks and tensors} 
\label{sec:shallow_networks}

Polynomial networks with $h=2$ are closely related to \emph{CP tensor decomposition}~\cite{Landsberg-book}.  Indeed in the shallow case, we can verify the network map $\Phi_{(d_0,d_1,d_2),r}$ sends $W_1 \in \RR^{d_1 \times d_0}, W_2 \in \RR^{d_2 \times d_1}$ to: 
\begin{equation}\label{eq:partiallySymTen}
W_2 \rho_r W_1 x \, = 
\, \Big{(} \sum_{i=1}^{d_1} W_{2}(:,i) \otimes W_{1}(i,:)^{\otimes r} \Big{)} \cdot x^{\otimes r} \, =: \, \Phi(W_2, W_1) \cdot x^{\otimes r}.
\end{equation}
Here $\Phi(W_2, W_1) \in \RR^{d_2} \times \Sym{r}{d_0}$ is a \textit{partially symmetric} $d_2 \times d_0^{\times r}$ tensor, expressed as a sum of $d_1$ partially symmetric rank $1$ terms, and $\cdot$ denotes contraction of the last $r$ indices. Thus the functional space $\mathcal{F}_{(d_0,d_1,d_2),r}$ is the set of rank $\leq d_1$ partially symmetric tensors. Algorithms for low-rank CP decomposition could be applied to $\Phi(W_2, W_1)$ to recover $W_2$ and $W_1$. In particular, when $d_2 = 1$, we obtain a symmetric $d_0^{\times r}$ tensor.  For this case, we have the following.

\begin{restatable}{lemma}{symTen}
A shallow architecture $\bm d = (d_0,d_1,1)$ is filling for the activation degree $r$ if and only if every symmetric tensor $T \in \Sym{r}{d_0}$ has rank at most $d_1$.
\end{restatable}

Furthermore, the celebrated \emph{Alexander-Hirschowitz Theorem}~\cite{alexander1995} from algebraic geometry provides the dimension of $\V{d}{r}$ for \emph{all} shallow, single-output architectures.

\begin{theorem}[Alexander-Hirschowitz]
If $\bm d = (d_0,d_1,1)$, the dimension of  $\V{d}{r}$ is given by $\min\left(d_0 d_1,\binom{d_0+r - 1}{r}\right)$, except for the following cases:
\begin{itemize}
    \item $r = 2$, $2 \le d_1 \le d_0-1$,
    \item $r = 3$, $d_0 = 5$, $d_1 = 7$,
    \item $r = 4$, $d_0 = 3$, $d_1 = 5$,
    \item $r = 4$, $d_0 = 4$, $d_1 = 9$,
    \item $r = 4$, $d_0 = 5$, $d_1 = 15$.
\end{itemize}
\end{theorem}

\subsection{Deep networks and tensors}
\label{sec:deep_tensors}

Deep polynomial networks also relate to a certain iterated tensor decomposition.  We first note the map $\Phi_{\bm d, r}$ may be expressed via the so-called \textit{Khatri-Rao product} from multilinear algebra.  Indeed $\theta$ maps to:
\begin{equation}\label{eq:khatri-rao}
\textup{SymRow}\,\, W_h((W_{h-1}\ldots (W_{2}(W_1^{\bullet r}))^{\bullet r}\ldots \, )^{\bullet r}).
\end{equation}
Here the Khatri-Rao product operates on rows: for $M \in \RR^{a \times b}$, the power $M^{\bullet r} \in \RR^{a \times b^r}$ replaces each row, $M(i,:)$, by its vectorized $r$-fold outer product, $\textup{vec}(M(i,:)^{\otimes r})$.  Also in \eqref{eq:khatri-rao}, SymRow denotes symmetrization of rows, regarded as points in $(\RR^{d_0})^{\otimes r^{h-1}},$ a certain linear operator.

Another viewpoint comes from using polynomials and inspecting the layers in reverse order.   Writing $[p_{\theta 1}, \ldots, p_{\theta d_{h-1}}]^{T}$ for the output polynomials at depth $h-1$, the top output at depth $h$ is:
\begin{equation}\label{eq:rpower}
w_{h11}\,p_{\theta 1}^{r} + w_{h12}\,p_{\theta 2}^{r} + \ldots + w_{h1d_{h-1}}\,p_{\theta d_{h-1}}^{r}. 
\end{equation}
This expresses a polynomial as a weighted sum of $r$-th powers of other (nonlinear) polynomials.  Recently, a study of such decompositions has been initiated in the algebra community~\cite{LORS-2019}.  Such expressions extend usual tensor decompositions, since weighted sums of powers of homogeneous \textit{linear} forms correspond to CP symmetric decompositions.
Accounting for earlier layers, our neural network expresses each $p_{\theta i}$ in~\eqref{eq:rpower} as $r$-th powers of lower-degree polynomials at depth $h-2$, so forth.  Iterating the main result in~\cite{FOS-PNAS} on decompositions of type \eqref{eq:rpower}, we obtain the following bound on filling intermediate widths.

\begin{table}[ht]
  \caption{Minimal filling widths for $r = 2$, $d_0=2$, $d_h=1$}
  \label{tab:minimal}
  \centering
\begin{tabular}{ccc}  
\toprule
Depth ($h$) & Degree ($r^h$)   & Minimal filling ($\bm d$) \\
\midrule
3      & 4 & (2,2,2,1)   \\
4         & 8 & (2,3,3,2,1)          \\
5       & 16  &(2,3,3,3,2,1)    \\
6       & 32 &(2,3,3,4,4,2,1)    \\
7 &      64  &(2,3,4,5,6,4,2,1)\\
8 & 128 &(2,3,4,5,7,7,6,2,1) or (2,3,5,5,7,7,5,2,1)\\
9 & 256 &(2,3,4,8,8,8,8,8,4,1) or (2,3,4,5,8,9,8,8,4,1)\\
\bottomrule
\end{tabular}
\end{table}
\begin{table}[ht]
    \centering
    \caption{Examples of dimensions of $\V{d}{r}$}
    \renewcommand{\arraystretch}{1.2}
    \begin{tabular}{l c c c c c}
    \toprule
         & $r = 2$  & $r = 3$ & $r = 4$ & $r = 5$ & $r = 6$\\
    \midrule
    $\bm d  = (3,2,1)$ & 5  & 6 & 6 & 6 & 6 \\
    $\bm d  = (2,3,2)$ & 6  & 8 & 9 & 9 & 9 \\
    $\bm d  = (2,3,2,3)$ & 10  & 12 & 13 & 13 & 13 \\
    $\bm d  = (2,3,2,3,4)$ & 16  & 21 & 22 & 22 & 22 \\
    \bottomrule
    \end{tabular}
    \label{tab:dimensions}
\end{table}

\begin{restatable}[Bound on filling widths]{theorem}{boundFilling}\label{thm:boundFilling}
Suppose $\bm{d}=(d_0,d_1,\ldots,d_h)$ and $r \geq 2$ satisfy 
\begin{equation}
d_{h-i} \, \geq \, \min \left(d_h \cdot r^{id_0}, \binom{r^{h-i} + d_{0} - 1}{r^{h-i}}\right)
\end{equation}
for each $i=1,\ldots,h-1$. Then the functional variety $\V{d}{r}$ is filling.
\end{restatable}

\subsection{Computational investigation of dimensions}
\label{sec:computations}

We have written code\footnote{Available at \url{https://github.com/mtrager/polynomial_networks}.} in the mathematical software \texttt{SageMath}~\cite{sagemath} that computes the dimension of $\V{d}{r}$ for a general architecture $\bm d$ and activation degree $r$. Our approach is based on randomly selecting parameters $\theta = (W_{h},\ldots,W_1)$ and computing the rank of the Jacobian of $\Phi_{\bm d,r}(\theta)$ in~\eqref{eq:phi_map}.  This method is based on the following lemma, coming from the fact that the map $\Phi_{\bm d, r}$ is algebraic.

\begin{restatable}{lemma}{jac}\label{lem:jac}
For all $\theta \in \RR^{d_{\theta}}$, the rank of the Jacobian matrix $\textup{Jac } \Phi_{\bm d, r}(\theta)$ is at most the dimension of the variety $\V{d}{r}$.  Furthermore, there is equality for almost all $\theta$ (\ie, for a non-empty Zariski-open subset of $\RR^{d_{\theta}}$). 
\end{restatable}

 Thus if $\textup{Jac } \Phi_{\bm d, r}(\theta)$ is full rank at any $\theta$, this witnesses a mathematical proof $\V{d}{r}$ is filling.  On the other hand if the Jacobian is rank-deficient at random $\theta$, this indicates with ``probability 1" that $\V{d}{r}$ is not filling.  We have implemented two variations of this strategy, by leveraging backpropagation:
\begin{enumerate}
    \item \emph{Backpropagation over a polynomial ring.} We defined a network class over a ring $\RR[x_1,\ldots,x_{d_0}]$, taking as input a vector variables $x = (x_1,\ldots,x_{d_0})$. Performing automatic differentiation (backpropagation) of the output function yields polynomials corresponding to $dp_\theta(x)/d w$, for any entry $w$ of a weight matrix $W_i$. Extracting the coefficients of the monomials in $x$, we recover the entries of the Jacobian of $\Phi_{\bm d,r}(\theta)$.
    \item \emph{Backpropagation over a finite field.} We defined a network class over a finite field $\FF = \ZZ/p\ZZ$. 
    After performing backpropagation at a sufficient number of random sample points $x$, we can recover the entries of the Jacobian of $\Phi_{\bm d,r}(\theta)$ by solving a linear system (this system is overdetermined, but it will have an exact solution since we use exact finite field arithmetic).  The computation over $\ZZ/p\ZZ$ provides the correct dimension over $\RR$ for almost all primes $p$. 
\end{enumerate}
The first algorithm is simpler and does not require interpolation, but is generally slower. We present examples of some of our computations in Tables~\ref{tab:minimal} and \ref{tab:dimensions}. Table~\ref{tab:minimal} shows minimal architectures $\bm d = (d_0,\ldots,d_h)$ that are filling, as the depth $h$ varies. Here, ``minimal'' is with respect to the partial ordering comparing all widths. It is interesting to note that for deeper networks, there is not a unique minimally filling network.  Also conspicuous is that minimal filling widths are ``unimodal", (weakly) increasing and then (weakly) decreasing.  Arguably, this pattern conforms with common wisdom.

\begin{conjecture}[Minimal filling widths are unimodal]
Fix $r$, $h$, $d_0$ and $d_h$. If $\bm{d}=(d_0,d_1,\ldots, d_h)$ is a minimal filling architecture, there is $i$ such that $d_0 \leq d_1 \leq \ldots \leq d_i$ and $d_i \geq d_{i+1} \geq \ldots \geq d_h$.
\end{conjecture}

Table~\ref{tab:dimensions} shows examples of computed dimensions, for varying architectures and degrees. Notice that the dimension of an architecture stabilizes as the degree $r$ increases.

\section{General results}

This section presents general results on the dimension of $\V{d}{r}$.  
We begin by pointing out symmetries in the network map $\Phi_{\bm{d}, r}$, under suitable scaling and permutation.

\begin{restatable}[Multi-homogeneity]{lemma}{multihom}\label{lem:multihom}
For all diagonal matrices $D_{i} \in \RR^{d_i \times d_i}$ and permutation matrices $P_i \in \ZZ^{d_i \times d_i}$ ($i=1,\ldots,h-1$), the map $\Phi_{\bm d, r}$ returns the same output under the replacement:
\begin{align*}
W_{1} & \leftarrow P_{1} D_{1} W_{1} \\
W_{2} & \leftarrow P_{2} D_{2} W_{2} D_{1}^{-r} P_{1}^{T} \\
W_{3} & \leftarrow P_{3} D_{3} W_{3} D_{2}^{-r} P_{2}^{T} \\
& \hspace{0.3cm} \vdots \\
W_h & \leftarrow W_{h} D_{h-1}^{-r} P_{h-1}^{T}.
\end{align*}
Thus the dimension of a generic fiber (pre-image) of $\Phi_{\bm d, r}$ is at least $\sum_{i=1}^{h-1} d_i$.
\end{restatable}

Our next result deduces a general upper bound on the dimension of $\V{d}{r}$.  Conditional on a  standalone conjecture in algebra, we prove that equality in the bound is achieved for all sufficiently high activation degrees $r$.  An unconditional result is achieved by varying the activation degrees per~layer.

\vspace{.3em}

\begin{restatable}[Naive bound and equality for high activation degree]{theorem}{naive}\label{thm:naive} If $\bm d = (d_0,\ldots,d_h)$, then
\begin{equation}\label{eq:naive}
    \dim \V{d}{r} \le \min\left(d_{h} + \sum_{i=1}^{h} (d_{i-1}-1) d_{i}, \, d_h \binom{d_0 + r^{h-1} -1}{r^{h-1}}\right).
\end{equation}
Conditional on Conjecture~\ref{conj:linIndep}, for fixed $\bm{d}$ satisfying $d_i > 1$  ($i=1,\ldots,h-1$), there exists $\tilde{r} = \tilde{r}(\bm{d})$ such that whenever $r > \tilde{r}$, we have an equality in \eqref{eq:naive}.
Unconditionally, for fixed $\bm{d}$ satisfying $d_i > 1$  ($i=1,\ldots,h-1$), there exist infinitely many $(r_{h-1}, r_{h-2}, \ldots, r_1)$ such that the image of  $(W_h,\ldots,W_1) \mapsto W_h \rho_{r_{h-1}} W_{h-1} \rho_{r_{h-2}} \ldots \rho_1 W_1 x$ 
has dimension $d_h + \sum_i (d_{i-1} -1)d_i$.
\end{restatable}

\vspace{.3em}

\begin{restatable}{proposition}{highDeg}\label{prop:highDeg}
Given integers $d, k, s$, there exists $\tilde{r}=\tilde{r}(d,k,s)$ with the following property.  Whenever $p_1, \ldots, p_k \in \RR[x_1, \ldots, x_d]$ are $k$ homogeneous polynomials of the same degree $s$ in $d$ variables, no two of which are linearly dependent, then $p_1^r, \ldots, p_k^r$ are linearly independent if $r > \tilde{r}$.
\end{restatable}

\vspace{.3em}

\begin{conjecture}\label{conj:linIndep}
In the setting of Proposition~\ref{prop:highDeg}, $\tilde{r}$ may be taken to depend only on $d$ and $k$.
\end{conjecture}

Proposition \ref{prop:highDeg} and Conjecture~\ref{conj:linIndep} are used in induction on $h$ for the equality statements in Theorem~\ref{thm:naive}. Our next result uses the iterative nature of neural networks to provide a recursive bound.

\vspace{.3em}

\begin{restatable}[Recursive Bound]{proposition}{recursive}\label{prop:recursive}
For all $(d_0, \ldots, d_k, \ldots, d_h)$ and $r$, we have:
\begin{equation}\label{eq:recursive}
\dim\mathcal{V}_{(d_0,\ldots,d_h),r} \, \leq \, \dim \mathcal{V}_{(d_0,\ldots,d_k),r} \, +\, \dim \mathcal{V}_{(d_k,\ldots,d_h),r} \, -\,  d_k.
\end{equation}
\end{restatable}

Using the recursive bound, we can prove an interesting \textit{bottleneck property} for polynomial networks.

\vspace{.3em}

\begin{definition}
The width $d_i$ in layer $i$ is an \textup{asymptotic bottleneck} (for $r$, $d_0$ and $i$) if there exists $\tilde{h}$ such that for all $h > \tilde{h}$ and all $d_1$, \ldots, $d_{i-1}$, $d_{i+1}, \ldots, d_{h}$, then the widths $(d_0, d_{1}, \ldots, d_{i}, \ldots, d_h)$ are non-filling.
\end{definition}

This expresses our finding that too narrow a layer can ``choke" a polynomial network, such that there is no hope of filling the ambient space, regardless of how wide elsewhere or deep the network is.

\vspace{.3em}

\begin{restatable}[Bottlenecks]{theorem}{bottleneck}\label{thm:bottleneck}
If $r \geq 2, d_0 \geq 2, i \geq 1$, then $d_i = 2d_0 - 2$ is an asymptotic bottleneck.  Moreover conditional on Conjecture 2 in~\cite{nicklasson-2017}, then $d_i = 2d_0 $  is not an asymptotic bottleneck.  
\end{restatable}

Proposition \ref{prop:recursive} affords a simple proof $d_i = d_0 - 1$ is an asymptotic bottleneck.  However to obtain the full statement of Theorem~\ref{thm:bottleneck}, we seem to need more powerful tools from algebraic geometry.

\section{Conclusion}

We have studied the functional space of neural networks from a novel perspective.  Deep polynomial networks furnish a framework for nonlinear networks, to which the powerful mathematical machinery of algebraic geometry may be applied. In this respect, we believe polynomial networks can help us access a better understanding of \textit{deep nonlinear architectures}, for which a precise theoretical analysis has been extremely difficult to obtain. 
Furthermore, polynomials can be used to approximate any continuous activation function over any compact support (Stone?Weierstrass theorem). For these reasons, developing a theory of deep polynomial networks is likely to pay dividends in building understanding of general neural networks.

In this paper, we have focused our attention on the \emph{dimension} of the functional space of polynomial networks. The dimension is the first and most basic descriptor of an algebraic variety, and in this context it provides an exact measure of the expressive power of an architecture. Our novel theoretical results include a general formula for the dimension of the architecture attained in high degree, as well as a tight lower bound and nontrivial upper bounds on the width of layers in order for the functional variety to be filling.  We have also demonstrated intriguing connections with tensor and polynomial decompositions, including some which appear in very recent literature in algebraic~geometry.

The tools and concepts introduced in this work for fully connected feedforward polynomial networks can be applied in principle to more general algebraic network architectures. Variations of our algebraic model could include multiple polynomial activations (rather than just single exponentiations) or more complex connectivity patterns of the network (convolutions, skip connections, \emph{etc.}). The functional varieties of these architectures could be studied in detail and compared. Another possible research direction is a geometric study of the functional varieties, beyond the simple dimension. For example, the \emph{degree} or the \emph{Euclidean distance degree}~\cite{draisma_euclidean_2013} of these varieties could be used to bound the number of critical points of a loss function. Additionally, motivated by Section~\ref{sec:deep_tensors}, we would like to develop computational methods for constructing a network architecture that represents an assigned polynomial mapping. Such algorithms might lead to ``closed form'' approaches for learning using polynomial networks (similar to SVD or tensor decomposition), as a provable counterpoint to gradient descent methods. Our research program might also shed light on the practical problem of choosing an appropriate architecture for a given application.

\section*{Acknowledgements}
We thank Justin Chen, Amit Moscovich, Claudiu Raicu and Steven Sam for their help.
JK was partially supported by the Simons Collaboration on Algorithms and Geometry.  MT and JB were partially supported by the Alfred P. Sloan Foundation, NSF RI-1816753 and Samsung Electronics.

\bibliographystyle{plain}

\begin{thebibliography}{10}

\bibitem{alexander1995}
James Alexander and Andr\'e Hirschowitz.
\newblock Polynomial interpolation in several variables.
\newblock {\em Journal of Algebraic Geometry}, 4(2):201--222, 1995.

\bibitem{arora_convergence_2018}
Sanjeev Arora, Nadav Cohen, Noah Golowich, and Wei Hu.
\newblock A convergence analysis of gradient descent for deep linear neural
  networks.
\newblock In {\em International Conference on Learning Representations}, 2019.

\bibitem{arora_optimization_2018}
Sanjeev Arora, Nadav Cohen, and Elad Hazan.
\newblock On the optimization of deep networks: implicit acceleration by
  overparameterization.
\newblock In {\em International Conference on Machine Learning}, pages
  244--253, 2018.

\bibitem{bisht}
Pranav Bisht.
\newblock On {{hitting sets}} for {{special depth}}-{{4 circuits}}.
\newblock Master's thesis, \textit{Indian Institute of Technology Kanpur},
  2017.

\bibitem{Blekherman-Teitler}
Grigoriy Blekherman and Zach Teitler.
\newblock On maximum, typical and generic ranks.
\newblock {\em Mathematische Annalen}, 362(3-4):1021--1031, 2015.

\bibitem{CMring-book}
Winfried Bruns and J\"{u}rgen Herzog.
\newblock {\em Cohen-{M}acaulay rings}, volume~39 of {\em Cambridge Studies in
  Advanced Mathematics}.
\newblock Cambridge University Press, Cambridge, 1993.

\bibitem{chizat_global_2018}
Lenaic Chizat and Francis Bach.
\newblock On the global convergence of gradient descent for over-parameterized
  models using optimal transport.
\newblock In {\em Advances in Neural Information Processing Systems}, pages
  3036--3046, 2018.

\bibitem{cohen_expressive_2016}
Nadav Cohen, Or~Sharir, and Amnon Shashua.
\newblock On the expressive power of deep learning: a tensor analysis.
\newblock In {\em Conference on Learning Theory}, pages 698--728, 2016.

\bibitem{cohen_convolutional_2016}
Nadav Cohen and Amnon Shashua.
\newblock Convolutional rectifier networks as generalized tensor
  decompositions.
\newblock In {\em International Conference on Machine Learning}, pages
  955--963, 2016.

\bibitem{cybenko_approximation_1989}
George Cybenko.
\newblock Approximation by superpositions of a sigmoidal function.
\newblock {\em Mathematics of Control, Signals and Systems}, 2(4):303--314,
  1989.

\bibitem{delalleau2011shallow}
Olivier Delalleau and Yoshua Bengio.
\newblock Shallow vs. deep sum-product networks.
\newblock In {\em Advances in Neural Information Processing Systems}, pages
  666--674, 2011.

\bibitem{sagemath}
The~Sage Developers.
\newblock {\em {S}ageMath, the {S}age {M}athematics {S}oftware {S}ystem
  ({V}ersion 8.0.0)}, 2017.
\newblock \url{http://www.sagemath.org}.

\bibitem{draisma_euclidean_2013}
Jan Draisma, Emil Horobe\c{t}, Giorgio Ottaviani, Bernd Sturmfels, and Rekha~R.
  Thomas.
\newblock The {E}uclidean distance degree of an algebraic variety.
\newblock {\em Foundations of Computational Mathematics}, 16(1):99--149, 2016.

\bibitem{du_power_2018}
Simon~S. Du and Jason~D. Lee.
\newblock On the power of over-parametrization in neural networks with
  quadratic activation.
\newblock In {\em International Conference on Machine Learning}, pages
  1329--1338, 2018.

\bibitem{Eisenbud-book}
David Eisenbud.
\newblock {\em Commutative algebra: with a view toward algebraic geometry},
  volume 150 of {\em Graduate Texts in Mathematics}.
\newblock Springer-Verlag, New York, 1995.

\bibitem{FOS-PNAS}
Ralf Fr{\"o}berg, Giorgio Ottaviani, and Boris Shapiro.
\newblock On the {W}aring problem for polynomial rings.
\newblock {\em Proceedings of the National Academy of Sciences},
  109(15):5600--5602, 2012.

\bibitem{harris_algebraic_1995}
Joe Harris.
\newblock {\em Algebraic geometry: a first course}, volume 133 of {\em Graduate
  Texts in Mathematics}.
\newblock Springer-Verlag, New York, corrected 3rd print edition, 1995.

\bibitem{Harshorne-book}
Robin Hartshorne.
\newblock {\em Algebraic geometry}, volume~52 of {\em Graduate Texts in
  Mathematics}.
\newblock Springer-Verlag, New York-Heidelberg, corrected 8th print edition,
  1997.

\bibitem{hornik_multilayer_1989}
Kurt Hornik, Maxwell Stinchcombe, and Halbert White.
\newblock Multilayer feedforward networks are universal approximators.
\newblock {\em Neural Networks}, 2(5):359--366, 1989.

\bibitem{kawaguchi_deep_2016}
Kenji Kawaguchi.
\newblock Deep learning without poor local minima.
\newblock In {\em Advances in Neural Information Processing Systems}, pages
  586--594, 2016.

\bibitem{Landsberg-book}
J.~M. Landsberg.
\newblock {\em Tensors: geometry and applications}, volume 128 of {\em Graduate
  Studies in Mathematics}.
\newblock American Mathematical Society, Providence, RI, 2012.

\bibitem{leshno_multilayer_1993}
Moshe Leshno, Vladimir~Ya. Lin, Allan Pinkus, and Shimon Schocken.
\newblock Multilayer feedforward networks with a nonpolynomial activation
  function can approximate any function.
\newblock {\em Neural Networks}, 6(6):861--867, 1993.

\bibitem{LORS-2019}
Samuel Lundqvist, Alessandro Oneto, Bruce Reznick, and Boris Shapiro.
\newblock On generic and maximal $k$-ranks of binary forms.
\newblock {\em Journal of Pure and Applied Algebra}, 223(5):2062 -- 2079, 2019.

\bibitem{martens_expressive_2014}
James Martens and Venkatesh Medabalimi.
\newblock On the expressive efficiency of sum product networks.
\newblock {\em {arXiv} preprint {arXiv}:1411.7717}, 2014.

\bibitem{mei_mean_2018}
Song Mei, Andrea Montanari, and Phan-Minh Nguyen.
\newblock A mean field view of the landscape of two-layer neural networks.
\newblock {\em Proceedings of the National Academy of Sciences},
  115(33):7665--7671, 2018.

\bibitem{nicklasson-2017}
Lisa Nicklasson.
\newblock On the {H}ilbert series of ideals generated by generic forms.
\newblock {\em Communications in Algebra}, 45(8):3390--3395, 2017.

\bibitem{poon_sum-product_2012}
Hoifung Poon and Pedro Domingos.
\newblock Sum-product networks: a new deep architecture.
\newblock {\em {arXiv} preprint {arXiv}:1202.3732}, 2012.

\bibitem{soltanolkotabi_theoretical_2018}
Mahdi Soltanolkotabi, Adel Javanmard, and Jason~D. Lee.
\newblock Theoretical insights into the optimization landscape of
  over-parameterized shallow neural networks.
\newblock {\em IEEE Transactions on Information Theory}, 65(2):742--769, 2019.

\bibitem{venturi2018a}
Luca Venturi, Afonso~S. Bandeira, and Joan Bruna.
\newblock Spurious valleys in two-layers neural network optimization
  landscapes.
\newblock {\em arXiv preprint arXiv:1802.06384}, 2018.

\end{thebibliography}

\appendix
\section{Technical proofs}

\CvR*
\begin{proof}
We mimic the proof of Theorem 1 in~\cite{Blekherman-Teitler}.
As $\mathcal{F}_{\bm{d}, r}$ is thick, equivalently $\mathcal{F}_{\bm{d}, r}$ contains some Euclidean open ball $B \subset \Sym{r^{h-1}}{d_0}^{d_h}$ (see Chevalley's theorem~\cite{Harshorne-book}). But given any point  $p \in \Sym{r^{h-1}}{d_0}^{d_h}$, we may write $p = \lambda_1 p_1 + \lambda_2 p_2$ for some $p_1, p_2 \in B$ and $\lambda_1, \lambda_2 \in \RR$.  Thus in the architecture $\bm{d}'$, we may set the ``top half" of weights to represent $p_1$, the ``bottom half" to represent $p_2$, and so scaling $W_h$ appropriately, all together the network represents $\lambda_1 p_1 + \lambda_2 p_2$.  
\end{proof}

\vspace{0.3cm}

\firstOpt*

\begin{proof}
We write $cl(C)$ for the closure of $C$. Let $L \subset \RR^n$ a line that intersects $cl(C)$ in (at least) two closed disjoint intervals $L \cap cl(C) \supset I_1 \cup I_2$. Such line always exists because $cl(C)$ is not convex. It is easy to construct a convex function $f: \RR^{n} \rightarrow \RR\cup \{+\infty\}$ that is $+\infty$ outside of $L$ and  has (arbitrarily) different minima when restricted to $I_1, I_2$: this amounts to constructing a convex function $\tilde f: \RR \rightarrow \RR$ with assigned minima on disjoint closed intervals.
\end{proof}

\vspace{0.3cm}

\secondOpt*
\begin{proof}
It is enough to argue that $\F{d}{r}$ does not lie on a linear subspace (\ie, that its affine hull is the whole ambient space). Indeed, because $\F{d}{r}$ has zero-measure, this implies that it cannot coincide with its convex hull. To show the claim, we observe that $\F{d}{r}$ always contains all vectors of polynomials of the form $q_i(\ell) = [0,\ldots0,\ell^{r^{h-1}},0,\ldots,0]^T \in \Sym{r^{h-1}}{d_0}^{d_h}$, where $\ell$ is a linear form in $d_0$ variables (this follows by induction on $h$). The vectors $q_i(\ell)$ span the ambient space, as any polynomial can be written as a linear combination of powers of linear forms. 
\end{proof}

\vspace{0.3cm} 

\symTen*
\begin{proof}
This is clear as the network outputs $\Phi(W_2,W_1) = \sum_{i=1}^{d_1} w_{21i} W_{1}(i,:)^{\otimes r} \in \Sym{r}{d_0}$.
\end{proof}
 
\vspace{0.3cm}

\boundFilling*
\begin{proof}
It is equivalent to show that the network map with scalars extended to $\CC$ (\ie, allowing complex weights), denoted $\Phi_{\bm{d}, r} \otimes \CC : \CC^{d_{\theta}} \rightarrow \textup{Sym}_{r^{h-1}}(\CC^{d_0})^{d_h}$, has full-measure image. For this, we use induction on $h$.  The key input is Theorem 4 of~\cite{FOS-PNAS}, which states generic homogeneous polynomials over $\CC$ of degree $rs$ in $d$ variables can be written as a sum of $\leq r^d$ many $r$-th powers of degree $s$ polynomials over $\CC$, when $r\geq 2$.  

The base case $h=1$ is trivial.  Thus assume $h>1$ and that the image has full measure for $h-1$.  If $d_{h-1} \geq \binom{r^{h-1}+d_0-1}{r^{h-1}}$, then for generic $W_{h-1},\ldots,W_{1}$, the entries of $\rho_r W_{h-1} \ldots \rho_r W_1 x$ form a vector space basis of $\textup{Sym}_{r^{h-1}}(\CC^{d_0})$, so the image of $\Phi_{\bm{d}, r} \otimes \CC$ is filling.  On the other hand if $d_{h-1} \geq d_h \cdot r^{d_0}$, then the image of $\Phi_{\bm{d}, r} \otimes \CC$ is full measure by \cite{FOS-PNAS} and the inductive hypothesis.
\end{proof}
 
 \vspace{0.3cm}
 
 \jac*
 \begin{proof}
 We note entries of $\textup{Jac }\Phi_{\bm{d}, r}(\theta)$ are polynomials in $\theta$, thus minors of $\textup{Jac }\Phi_{\bm{d}, r}(\theta)$ are polynomials in $\theta$, so $\textup{Jac }\Phi_{\bm{d}, r}(\theta)$ has a Zariski-generic rank (the largest size of minor that is a nonzero polynomial), which is also the maximum rank of $\textup{Jac }\Phi_{\bm{d}, r}(\theta)$.   By basic algebraic geometry, this is the dimension of $\mathcal{V}_{\bm{d}, r}$ (see ``generic submersiveness" of algebraic maps in characteristic 0~\cite{Harshorne-book}).
 \end{proof}
 
 \vspace{0.3cm}
 
 \multihom*
 \begin{proof}
This is from the multi-homogeneity of the $r$-th power activation $\rho_r$ by substituting.
 \end{proof}
 
 \vspace{0.3cm}
 
 \naive*
 \begin{proof}
We know the dimension of $\V{d}{r}$ equals the dimension of the domain of $\Phi_{\bm{d}, r}$ minus the dimension of a generic fiber of $\Phi_{\bm{d}, r}$ (see generic freeness~\cite{Eisenbud-book}).  Thus by Lemma~\ref{lem:multihom}, $\dim \V{d}{r} \leq \sum_{i=1}^{h} d_{i-1} d_{i} \, - \, \sum_{i=1}^{h-1}d_i = d_h + \sum_{i=1}^{h} (d_{i-1} - 1)d_i$.  At the same time, the dimension of $V_{\bm{d}, r}$ is at most that of its ambient space $\Sym{r^{h-1}}{d_0}^{d_h}$.  Combining produces the bound~\eqref{thm:boundFilling}.

For the next statement, we temporarily assume Conjecture~\ref{conj:linIndep}.  We shall prove by induction on $h$ the stronger result that for $r \gg 0$ the generic fibers of $\Phi_{\bm{d}, r}$ are precisely as described in Lemma~\ref{lem:multihom} (and no more).
The base case $h=1$ is trivial.  Thus assume $h>1$ and that for $h-1$ the generic fiber is exactly 
as in Lemma~\ref{lem:multihom}, whenever $r > \tilde{r}_{1} = \tilde{r}_{1}(d_0,\ldots,d_{h-1})$.  
For the induction step, we let $\tilde{r}_2 = \tilde{r}_{2}(d_0, d_{h-1})$ be a threshold which works in Conjecture~\ref{conj:linIndep} for $d=d_0$ and $k=2d_{h-1}$, and then we set $\tilde{r}_{3} = \tilde{r}_{3}(d_0,\ldots,d_h) = \max(\tilde{r}_{1}, \tilde{r}_{2})$. Now with fixed generic weights $W_h, \ldots, W_1$, we consider any 
other weights $\tilde{W}_h, \ldots, \tilde{W}_h$ satisfying
\begin{equation}\label{eq:fiber}
W_h \rho_r W_{h-1} \ldots \rho_r W_1 x = \tilde{W}_h \rho_r \tilde{W}_{h-1} \ldots \rho_r \tilde{W}_1 x
\end{equation}
for $r > \tilde{r}_3$.  Write $\begin{bmatrix}p_{\theta 1} & \ldots & p_{\theta d_{h-1}} \end{bmatrix}$ for the output of the LHS in~\eqref{eq:fiber} at depth $h-1$, and similarly $\begin{bmatrix}\tilde{p}_{\theta 1} & \ldots & \tilde{p}_{\theta d_{h-1}} \end{bmatrix}$ for the RHS.
By genericity and $d_i > 1$, the polynomials $p_{\theta i}$ are pairwise linearly independent.  Comparing the top outputs at depth $h$ in~\eqref{eq:fiber}, we get two decompositions of type~\eqref{eq:rpower}:
\begin{equation}\label{eq:twoDecomp}
w_{h11} p_{\theta 1}^{r} + \ldots + w_{h1d_{h-1}} p_{\theta d_{h-1}}^{r} = \tilde{w}_{h11} \tilde{p}_{\theta 1}^{r} + \ldots + \tilde{w}_{h1d_{h-1}} \tilde{p}_{\theta d_{h-1}}^{r}.
\end{equation}
Since $r>\tilde{r}_2$, by Conjecture~\ref{conj:linIndep} there must be two linearly dependent summands in~\eqref{eq:twoDecomp}. Permuting as necessary we may assume these are the first two terms on both sides.  Scaling as necessary we may assume $p_{\theta 1} = \tilde{p}_{\theta 1}$, and then subtract $\tilde{w}_{h11}\tilde{p}_{\theta 1}^r$ from~\eqref{eq:twoDecomp} to get:
\begin{equation}\label{eq:twoDecompv2}
(w_{h11}-\tilde{w}_{h11}) p_{\theta 1}^{r} + \ldots + w_{h1d_{h-1}} p_{\theta d_{h-1}}^{r} = \tilde{w}_{h12} \tilde{p}_{\theta 2}^{r} + \ldots + \tilde{w}_{h1d_{h-1}} \tilde{p}_{\theta d_{h-1}}^{r}.
 \end{equation}
 Invoking Conjecture~\ref{conj:linIndep} again, we may remove another summand from the RHS, so on until the RHS is 0.  Then each individual summand in the LHS must be 0 too, by pairwise linear independence and Conjecture~\ref{conj:linIndep} once more.  We have argued that (up to scales and permutation) it must hold $\begin{bmatrix}p_{\theta 1} & \ldots & p_{\theta d_{h-1}} \end{bmatrix} = \begin{bmatrix}\tilde{p}_{\theta 1} & \ldots & \tilde{p}_{\theta d_{h-1}} \end{bmatrix}$ and $W_{h}(1,:) = \tilde{W}_{h}(1,:)$.  Comparing other outputs at depth $h$ in~\eqref{eq:fiber} gives $W_h = \tilde{W}_h$ (up to scales and permutation).  Thus by the inductive hypothesis, the fiber through $(W_h, \ldots, W_1)$ is as in Lemma~\ref{lem:multihom} and no more.  This completes the induction.
 
For the unconditional result with differing degrees per layer, the argument runs closely along similar lines, but it relies on Proposition~\ref{prop:highDeg} in place of Conjecture~\ref{conj:linIndep}.  For brevity, details are omitted.
 \end{proof}
 
 \vspace{0.3cm}
 
\highDeg*
\begin{proof}
It is shown in~\cite{bisht} (via Wronskian and Vandermonde determinants) that for any \textit{particular} $p_1, \ldots, p_k$, no two of which are linearly dependent, there exists $\tilde{r} = \tilde{r}(p_1, \ldots, p_k)$ such that $p_1^r, \ldots, p_k^r$ if $r > \tilde{r}$.  The dependence on particular $p_1, \ldots, p_k$ can be removed as follows.

Let $U \subset \Sym{s}{d}^k$ be the set of $k$-tuples, no two entries of which are linearly dependent.  So $U$ is Zariski-open, described by the non-vanishing of $2\times2$ minors.  Further let $U_r \subseteq U$ be the subset of $k$-tuples whose $r$-th powers are linearly independent, similarly Zariski-open.  Consider the chain of inclusions $U_1 \, \subseteq \, U_1 \cup U_2 \, \subseteq \, U_1 \cup U_2 \cup U_3 \, \subseteq \, \ldots$.  By~\cite{bisht}, the union of this chain equals $U$.  Thus by Noetherianity of affine varieties, there exists $R$ with $\cup_{r=1}^{R} U_{r} \, = \, U$~\cite{Eisenbud-book}.  Now $\tilde{r} = R!$ works.
\end{proof}

\vspace{0.3cm}

\recursive*
\begin{proof}
This bound encapsulates the bracketing:
\begin{equation}\label{eq:factor}
(W_h \rho_r W_{h-1} \ldots W_{k+1}) \rho_r (W_k \rho_r W_{k-1} \ldots W_1 x).     
\end{equation}
More formally, the network map $\Phi_{(d_0,\ldots,d_h),r}$ factors as:
\begin{equation}\label{eq:factor2}
\RR^{d_{\theta}} \,\, \longrightarrow \,\, \Sym{r^{h-k-1}}{d_k}^{d_h}  \times \ \Sym{r^{k-1}}{d_0}^{d_k} \,\, \longrightarrow \,\, \Sym{r^{h-1}}{d_0}^{d_h}   
\end{equation}
by first sending $(W_h, \ldots, W_1)$ to the pair of bracketed terms in \eqref{eq:factor} and then the pair to the composite in~\eqref{eq:factor}.  The closure of the image of the first map in~\eqref{eq:factor2} is $\mathcal{V}_{(d_0,\ldots,d_h),r} \times \mathcal{V}_{(d_k,\ldots,d_h),r}$.  On the other hand, the second map in~\eqref{eq:factor2} has $\geq d_{k}$-dimensional generic fibers, by multiplying with a diagonal matrix $D_{k} \in \RR^{d_{k}\times d_{k}}$.  Combining these facts gives the result.
\end{proof}

\vspace{0.3cm}

\bottleneck*
\begin{proof}
We first point out that Proposition~\ref{prop:recursive} gives an elementary proof $d_i = d_0 - 1$ is an asymptotic bottleneck.  This is because as $h$ grows the ambient dimension grows like $O(d_h \cdot d_{0}^{r^{h-1}})$, while the RHS bound grows like $O(d_h \cdot d_{i}^{r^{h-i-1}})$, so if $d_i < d_0$ then $\V{d}{r}$ cannot fill for $h \gg 0$. 

To gain a factor of 2 in the bottleneck bound, we start by writing $\begin{bmatrix} p_{\theta 1} & \ldots & p_{\theta d_i} \end{bmatrix}^{T}$ for the output polynomials at depth $i$, that is, for $W_i \rho_r W_{i-1} \ldots \rho_r W_1 x$.  Fixing $\theta$,   we consider $A_{\theta} := \RR[p_{\theta 1}^{r}, \ldots, p_{\theta d_i}^{r}]$, a subalgebra of the \textit{Veronese ring} $V_{d_0, r^{i}} := \RR [x_{1}^{\alpha_1}  \ldots  x_{d_0}^{\alpha_{d_0}}  :  \sum_{j=1}^{d_0} \alpha_j = r^i ]$.  The key idea is to compare the \textit{Hilbert polynomials} of $A_{\theta}$ and of $V_{d_0, r^i}$~\cite{CMring-book}.  If the Hilbert polynomials \textit{differ in any non-constant terms}, this means the dimension of the degree $D$ piece of $A_\theta$ minus that of $V_{d_0, r^i}$ diverges to $- \infty$ as $D$ goes to $\infty$.  At the same time, however we vary weights $W_{i+1}, \ldots, W_h$ (keeping $\theta = W_1,\ldots,W_i$ fixed), the output polynomials $\Phi_{\bm{d}, r}$ remain in the algebra $A_{\theta}$.  Additionally, for varying $\theta$ and $d_1, \ldots, d_{i-1}$, the possible $d_i$-vectors of degree $r^i$ polynomials in $d_0$ variables, $\begin{bmatrix} p_{\theta 1} & \ldots & p_{\theta d_i} \end{bmatrix}^{T}$, comprise a bounded-dimensional variety.  The upshot is that if it need always be the case (based on $r, d_0, i, d_i$) that the Hilbert polynomials of $A_{\theta}$ and $V_{d_0, r^i}$ have non-constant difference, then $d_i$ must be an asymptotic bottleneck.  Thus it suffices to check the Hilbert polynomial property holds for all $\theta$  if $d_i = 2d_0 - 2$.
To this end, we derived the following general result:  

\begin{claim*}\label{claim1}
Given integers $d \geq 2$ and $s\geq 2$.  Then whenever $p_1, \ldots, p_{2d-2} \in \RR[x_1, \ldots, x_d]$ are $2d-2$ homogeneous polynomials of the same degree $s$ in $d$ variables, the algebra $\RR[p_1, \ldots, p_{2d-2}]$ and the Veronese algebra $V_{d,s}$ have Hilbert polynomials with non-constant difference.
\end{claim*}
\begin{proof}[Proof of claim]
First, it suffices to check the claim for generic $p_i$.
Second, the difference in Hilbert polynomials identifies with the Hilbert polynomial of the \textit{sheaf} $\mathcal{G} = \textup{coker}(\mathcal{O}_{Y} \rightarrow \pi_{\ast} \mathcal{O}_{X})$~\cite{Harshorne-book}.  Here $X := \mathbb{V}_{d,s} \subset \PP^{N_{s,d}-1}$ ($N_{s,d} = \binom{d+s-1}{s}$) is the \textit{projective Veronese variety}, the linear projection $\pi: \PP^{N_{s,d}-1} \dashrightarrow \PP^{2d-3}$ corresponds to $(p_1, \ldots, p_{2d-2})$, and finally $Y := \overline{\pi(X)}$ is the closure of $X$ projected by $\pi$.  By general facts, the degree of the Hilbert polynomial of $\mathcal{G}$ equals the projective dimension of the support of $\mathcal{G}$, and this support is the \textit{branch locus} of $\pi \vert_{X}$.  Now let $L \subset \PP^{N_{s,d}-1}$ denote the base locus (kernel) of $\pi$, a linear subspace of projective dimension $N_{s,d} - 2d +1$.  If $d \geq 3$, $s \geq 3$, then $L \cap Sec(X)$ is a curve, where $Sec$ denotes the line \textit{secant variety}~\cite{Landsberg-book} ($d=2$ or $s=2$ are omitted simple special cases).  Each point on  $L \cap Sec (X)$ lies on a line through two points on $X$; these points map to the same image under $\pi$, giving a point in the branch locus of $\pi \vert_{X}$.  It follows the branch locus is a curve, thus the degree of the Hilbert polynomial of $\mathcal{G}$ is $1 > 0$.
\end{proof}

By the preceding discussion, the claim establishes $d_i = 2d_0 - 2$ is an asymptotic bottleneck.  

For the statement when $d_i = 2d_0$, let us temporarily assume Conjecture 2 in~\cite{nicklasson-2017}.  This means $\RR[p_{\theta 1}^r,\ldots,p_{\theta 2d_{0}}^r]$ has the same Hilbert function as $\RR[P_1, \ldots, P_{2d_0}]$ for generic forms $P_i$ of degree $r^i$, provided $p_{\theta i}$ are generic forms of degree $r^{i-1}$.  Reasoning as for the claim, $\RR[P_1, \ldots, P_{2d_0}]$ has the same Hilbert polynomial as the Veronese ring $V_{d_0, r^i}$.  Thus if we choose $(d_1, \ldots, d_{i-1})$ so that $(d_0,\ldots,d_{i})$ is filling, then it follows we can choose $h \gg 0$ and $(d_{i+1},\ldots,d_{h})$ so that $(d_0, \ldots, d_h)$ is filling.  In other words, $d_i = 2d_0$ is not an asymptotic bottleneck.
\end{proof}

\end{document}